\documentclass{llncs}

\usepackage{times}
\usepackage{amssymb,amsmath}
\usepackage{mathrsfs} 
\usepackage{url}
\usepackage{graphicx}
\usepackage[ruled,vlined,linesnumbered]{algorithm2e}
\DontPrintSemicolon
\usepackage[table]{xcolor}
\usepackage{booktabs}
\usepackage{colortbl}
\usepackage{multirow}
\usepackage{subfigure}

\definecolor{vlgray}{gray}{0.92}
\definecolor{lgray}{gray}{0.88}
\definecolor{grey1}{rgb}{0.4,0.4,0.4}

\newcommand{\m}[1]{\mathit{#1}}

\newcommand{\p}[1]{\textit{#1}}

\begin{document}

\pagestyle{empty}

\mainmatter

\title{Solving WCSP \\by Extraction of Minimal Unsatisfiable Cores}

\titlerunning{Lecture Notes in Computer Science}

\author{Christophe Lecoutre \and Nicolas Paris \and Olivier Roussel \and S\'ebastien Tabary}

\institute{
CRIL - CNRS, UMR 8188, \\
Univ Lille Nord de France, Artois\\
F-62307 Lens, France\\
\{lecoutre,paris,roussel,tabary\}@cril.fr }

\maketitle

\begin{abstract}

Usual techniques to solve WCSP are based on cost transfer operations coupled with a branch and bound algorithm. 
In this paper, we focus on an approach integrating extraction and relaxation of Minimal Unsatisfiable Cores in order to solve this problem.
We decline our approach in two ways: an incomplete, greedy, algorithm and a complete one.

\end{abstract}

\section{Introduction}

The Constraint Satisfaction Problem (CSP) consists in determining if a Constraint Network (CN), also called a CSP instance, is satisfiable or not.
It is a decision problem, the goal of which is to find a solution, i.e., a complete instantiation of the variables of the CN that satisfies each constraint.
Sometimes, preferences among solutions need to be expressed. This is possible through the introduction of soft constraints. 
The framework called WCSP (Weighted CSP) allows us to handle such constraints: each soft constraint is defined by a cost function that associates a violation degree, called cost, with every possible instantiation of a subset of variables.
The WCSP goal is 
to find a complete instantiation with a minimal combined cost of the soft constraints.

Most of the current methods to solve Weighted CNs (WCNs) are based on branch and bound tree search together with the use of soft local consistencies 
for estimating minimal costs of sub-problems during search.
These local consistencies, e.g., AC* \cite{L_node,LS_solving}, FDAC \cite{C_reduction}, EDAC \cite{GHZL_existential}, VAC \cite{CGSSZ_virtal}, and OSAC \cite{CGSSZW_soft}, have been developed on the concept of equivalence-preserving transformations (cost transfer operations).
Because they must combine costs, algorithms establishing soft local consistencies are often more complex than their CSP counterparts.
In this paper, we propose an original approach for WCSP that consists in solving a WCN by iteratively  generating and solving classical CNs. 
One advantage is to benefit from efficient CSP solvers developed in the community for more than two decades. 
Notice that reformulating WCNs into CNs has already been used quite successfully to enforce VAC \cite{CGSSZ_virtal}, where (classical) arc consistency is established on a CN generated from a WCN in order to identify some relevant cost transfer operations.

In our approach, to solve a WCN, soft constraints are converted into hard constraints by retaining only the tuples that have a specified cost.
Actually, the principle is to enumerate a sequence of CNs according to an increasing cost order related to the WCN. 
When a CN is proved to be unsatisfiable, a Minimal Unsatisfiable Core (MUC) is extracted in order to identify the soft constraints whose reformulations into hard constraints need to be relaxed, so as to finally obtain a solution.
This approach is declined in two versions: a complete algorithm and a greedy one (that is incomplete). 
Note that related approaches have been successfully exploited for MaxSAT \cite{ABL_NAWPM,FM_SPMP,SP_AMSUC}.

This paper is organized as follows. After some usual definitions, we present the principle of our approach and the required data structures.
Next, we completely describe a complete version of our approach, as well as a greedy one.
Finally, after presenting some practical results, we conclude.

\section{Technical Background \label{sec:technical}}

\subsection{Generalities}

A {\em constraint network} (CN) $P$ involves a finite set of $n$ variables denoted by $vars(P)$ and a finite set of $e$ constraints denoted by $cons(P)$.
Each variable $x$ has an associated domain, denoted by $dom(x)$, which contains the finite set of values that can be assigned to $x$; the initial domain of $x$ is denoted by $dom^{init}(x)$.
Each (hard) constraint $c$ involves an ordered set of variables, denoted $scp(c)$ and called {\em scope} of $c$, and is defined by a relation containing the set of tuples allowed for the variables of $scp(c)$.
The arity of a constraint $c$ is $|scp(c)|$.
An {\em instantiation} $I$ of a set $X = \{x_1,\dots,x_p\}$ of variables is a set $\{(x_1,a_1)$, $\dots$, $(x_p,a_p)\}$ such that $\forall i \in 1..p, a_i \in dom^{init}(x_i)$. 
$I$ is {\em valid} on $P$ iff $\forall (x,a) \in I, a \in dom(x)$. 
A solution is a complete instantiation of $vars(P)$ (i.e., the assignment of a value to each variable) such that all the constraints are satisfied.
$P$ is satisfiable iff it admits at least one solution.
For more information on constraint networks, see \cite{D_book,L_constraint,RBW_handbook}.

An unsatisfiable core of $P$ is an unsatisfiable subset of constraints of $P$.
A core $C$ is a MUC (Minimal Unsatisfiable Core) iff each strict subset of constraints of $C$ is satisfiable.
Several methods to extract MUCs of constraint networks are presented in \cite{SP_explanation,J_preferred,HLSB_extracting}.
Among the different versions presented in \cite{HLSB_extracting}, we have chosen for our implementation the \textit{dichotomic} version, called \textit{dcMUC}.

A {\em weighted constraint network} (WCN) $W$ involves a finite set of $n$ variables denoted by $vars(W)$ and a finite set of $e$ soft constraints denoted by $cons(W)$, and has an associated value $k > 0$ which is either a natural integer or $+ \infty$.
Each soft constraint $w \in cons(W)$ has a scope $scp(w)$ and is defined as a cost function from $l(scp(w))$ to $\{0,\dots,k\}$, where $l(scp(w))$ is the Cartesian product of the domains of the variables involved in $w$; for each instantiation $I \in l(scp(w))$, the cost of $I$ in $w$ is denoted by $w(I)$.

When an instantiation is given the cost $k$ it is said {\em forbidden}. Otherwise, it is permitted with the corresponding cost (0, being completely satisfactory). 

Costs are combined with the bounded addition $\oplus$ defined as: 
$$\forall a,b \in \{0,\dots,k\}, a \oplus b = min(k,a+b)$$

The objective of Weighted Constraint Satisfaction Problem (WCSP) is, for a given WCN, to find a complete instantiation with a minimal cost.
For more information on weighted constraint networks, see \cite{BMRSVF_semiring,MRS_soft}.

Different variations of soft arc consistency for WCSP have been proposed during the last decade: AC* \cite{L_node,LS_solving}, full directional arc consistency (FDAC) \cite{C_reduction}, existential arc consistency (EDAC) \cite{GHZL_existential}, virtual arc consistency (VAC) \cite{CGSSZ_virtal} and optimal soft arc consistency (OSAC) \cite{CGSSZW_soft}.
All the algorithms proposed to enforce these different levels of consistency use cost transfer operations (based on the concept of equivalence-preserving transformations) such as unary projection, projection and extension.

\subsection{Layers and Fronts}

In this paper, we focus on extensional soft constraints. 
These are soft table constraints where some tuples are explicitly listed with their costs in a table, and a default cost indicates the cost of all implicit tuples (i.e., those not present in the table) ; e.g., see \cite{LPRT_propagating}.
Below, we introduce both the notions of layer and front.

All tuples of a soft table constraint having the same cost can be grouped to form a subset called {\em layer}.
The cost of any tuple from layer $L$ is given by $cost(L)$.
A particular layer corresponds to the default cost: this layer contains no tuple, but implicitly represents all tuples that do not explicitly appear elsewhere (in another layer).
The number of layers for a constraint $w$ is given by $nbLayers(w)$.
Within a constraint, we consider that layers are increasingly ordered according to their costs.
Finally, for the sake of simplicity, we shall use indices, from 0 to $nbLayers(w)-1$, to identify the different layers of a constraint, and when the context is unambiguous, we shall not distinguish between indices and the layers they represent.

A {\em front} $f$ is a function that maps each constraint of a WCN to one of its layers. 
A front represents a kind of border between the layers that will be considered (allowed) at a given moment, and the layers that will be discarded (forbidden).
We shall use an array notation for the fronts. 
So, $f[w]$ will represent the layer associated with the constraint $w$ in the front $f$. 
The cost of a front $f$, $cost(f)$, is obtained by summing up all costs of the constraint layers corresponding to the front $f$~: 
$$cost(f)=\sum_{w\in cons(W)}cost(f[w])$$

A front $f'$ is the direct successor of a front $f$ iff there exists a constraint $w_i$ such that $f'[w_i]=f[w_i]+1$ and $\forall j \ne i, f'[w_j]=f[w_j]$.
We note $f\rightarrow f'$ iff $f'$ is a direct successor of $f$ and $f\rightarrow_* f'$ iff there exists a sequence of relations $f\rightarrow f_1, f_1\rightarrow f_2,\ldots f_n\rightarrow f'$ (transitive closure).

Because the layers of any constraint are increasingly ordered following their costs, we deduce that:
$$f\rightarrow_* f'\Rightarrow cost(f)<cost(f')$$

We can observe that the set of all possible fronts and the direct successor relation form a lattice structure, where the lowest front (also noted $\bot$) associates with each constraint its first layer (of lowest cost) and the highest front (also noted $\top$) associates with each constraint its last layer (of highest cost).

\section{General Principle}

Figure \ref{fig:schema} illustrates the main idea of our approach for solving weighted constraint networks.
First, from a WCN $W$, an initial constraint network $P$ is built as follows:
for each constraint of $W$, the tuples occurring in the layer of minimal cost are considered as allowed in $P$, while other tuples are considered as forbidden (function $\m{toCN}$).
This process is similar to that proposed in \cite{CGSSZ_virtal}.
Next, the CN $P$ is solved using a constraint solver (through a call to the function $\m{solveCN}$).
If $P$ is satisfiable, then a solution is returned.
Otherwise a MUC is extracted from $P$ (through a call to the function $\m{extractMUC}$) and then exploited either in a greedy or complete version.

In the greedy version, some constraints of the MUC will be successively relaxed until satisfiability of the MUC is restored (through a call to the function $\m{relax}$).
More precisely, a subset of layers of some constraints belonging to the MUC will be switched from the ``forbidden'' status to the ``allowed'' one.
When constraints corresponding to the MUC are relaxed so as to become satisfiable, the process loops until the global satisfiability of $P$ is reached.
In this case, a solution is returned.

In the complete version, once a MUC is identified, we insert in a priority queue fronts that represent all possible ways of relaxing the MUC. 
Fronts are extracted from the priority queue, translated to a CN, and successively solved until the occurrence of a satisfiable constraint network.
In this case, an optimum solution is returned.

\begin{figure}
  \centerline{\includegraphics[width=6cm]{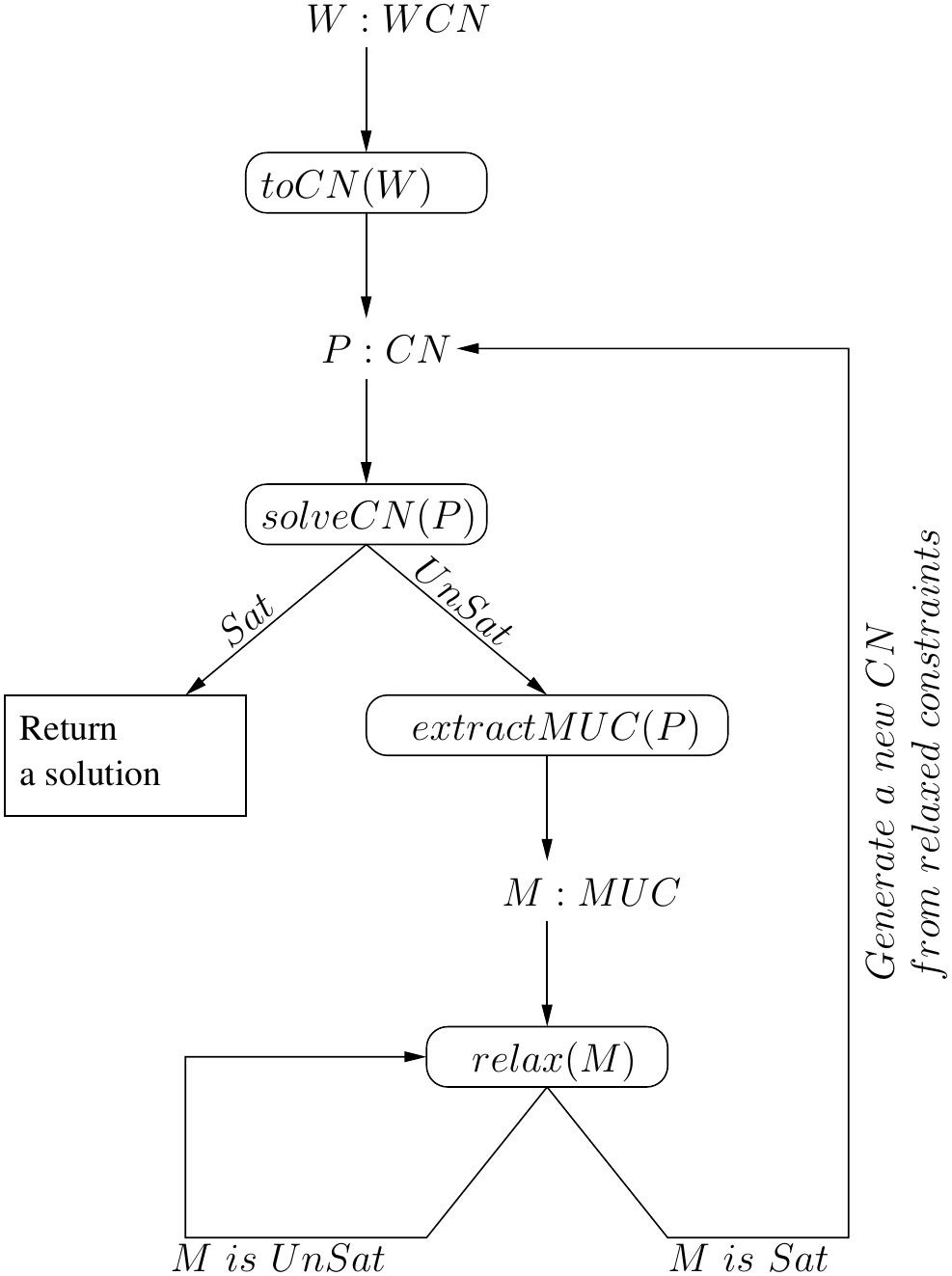}}
  \caption{\label{fig:schema} Principle of the iterative relaxation of MUCs.}
\end{figure}

Interestingly, the principle described above can be generalized to take into account any kind of constraint: hard constraints from the CSP framework, ``violable '' constraints from the Weighted Max-CSP framework and soft constraints from the WCSP framework.
When the function $\m{toCN}$ is called, we just need to indicate the translation process for each kind of constraints.
For a hard constraint, no transformation is required (we have only one layer with a cost equal to $0$).
For a ``violable'' constraint, two layers are managed.
The first layer, with a cost equal to $0$, corresponds to the satisfaction of the ``violable'' constraint.
The second one, with a cost equal to the weight of the  ``violable'' constraint, corresponds to the violation of the constraint.
In other words, our approach allows us to deal easily and naturally with these different frameworks. 
Due to lack of space, this is not further described in this paper.

\section{Data structures and Functions}

Fronts are exploited by the function $\m{toCN}$ which builds a CN from a WCN. 
Two versions are considered.
The first one, noted $\m{toCN}_=(W,f)$, builds a CN by selecting as allowed tuples for a soft constraint $w$ those belonging to the layer $f[w]$ ; this layer is said allowed.
The second one, noted $\m{toCN}_\le(W,f)$, builds a CN by selecting as allowed tuples for a soft constraint $w$ the tuples of the layers whose index is less than or equal to $f[w]$ ; these layers are said allowed.
In other words, for each hard constraint $c$ built from a soft constraint $w$ of $W$, the tuples allowed in $c$ are those having a cost equal (resp., less than or equal) to the cost of the layer $f[w]$.

In practice, two representations of a hard constraint can be envisionned.
On the one hand, when the layer corresponding to the default cost is not allowed, $\m{toCN}$ generates a hard positive constraint.
The tuples of the hard constraint are allowed and correspond to the tuples of the allowed layers of the soft constraint.
On the other hand, when the layer corresponding to the default cost is allowed, $\m{toCN}$ generates a hard negative constraint.
The tuples of the hard constraint are forbidden and correspond to the tuples of the forbidden layers of the soft constraint in order to avoid to enumerate the implicit tuples having a cost equal to the default cost.

Figure \ref{fig:structure} presents two constraints $w_0$ and $w_1$ with their layers.
Constraint $w_0$ has a set of layers numbered from $0$: layer $0$ of minimal cost $cost_0$ contains the tuples $\tau_1$, $\tau_5$, and $\tau_7$, and the next layer (layer $1$) of cost $cost_1$ contains the tuples $\tau_2$ and $\tau_3$. 
In this example, $f[w_0]$ is equal to $1$ and $f[w_1]$ is equal to $0$.
The CN returned by $toCN_=(W,f)$ contains a constraint $c_0$ (associated with $w_0$) with allowed tuples $\tau_2$ and $\tau_3$ and a constraint $c_1$ (associated with $w_1$)  with allowed tuples $\tau_1$ and $\tau_4$. 
The CN returned by $toCN_{\le}(W,f)$ contains a constraint $c_0$ (associated with $w_0$) with allowed tuples $\tau_2$, $\tau_3$, $\tau_1$, $\tau_5$, and $\tau_7$ and a constraint $c_1$ (associated with $w_1$)  with allowed tuples $\tau_1$ and $\tau_4$. 
Allowed layers of $toCN_{\le}(W,f)$ are identified by a grey background in Figure \ref{fig:structure}.


We conclude this section by a short description of the other functions we need.
The function $\m{solveCN(P)}$ solves a CN given in parameter and returns either a solution or $\bot$ if the CN is unsatisfiable.
The function $\m{extractMUC(P)}$ returns a MUC from an unsatisfiable CN $P$ given in parameter.
The function $\m{cons}(W,M)$ returns the set of soft constraints of the WCN $W$ which correspond to those of the MUC $M$ (by construction any hard constraint is associated with a soft constraint).
The function $\m{restrict}(W,M)$ returns a WCN containing only the constraints of the WCN $W$ corresponding to the constraints present in the MUC $M$.

\begin{figure}
   \centerline{ \includegraphics[width=5.5cm]{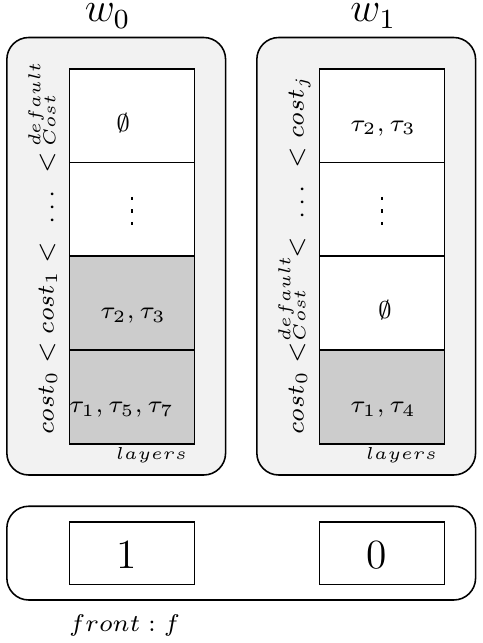}}
  \caption{\label{fig:structure} Description of data structures.}
\end{figure}


\section{Algorithms}

\subsection{Introductory Complete Version}

To facilitate the understanding of our approach, we present a first complete algorithm which does not exploit MUCs and thus is not efficient.
Function \ref{csnomuc} starts with a first front which retains only layer 0 of each constraint (lines $1$-$3$).
Next, while there exists a front to explore, we extract from a priority queue the front with the minimal cost.
The cost of a front is easily computed by summing the costs associated with the layers selected in each soft constraint. 
Then, we can build with the function $\m{toCN}_=$ a CN that we try to solve with $\m{solveCN}$.
If a solution is found, it is optimal and the algorithm stops.
Indeed, the algorithm enumerates fronts in increasing cost order (see Proposition \ref{queueproperty}) and guarantees that each previous front was unsatisfiable.
If the tested CN was unsatisfiable, we enumerate the direct successors of the current front (lines $11$-$16$) and insert them in the priority queue.
Obviously, we have to check that we are not exceeding the maximal layer for a given constraint.
Besides, fronts having a cost equal to $k$ must be ignored.

\begin{function}[h!]
  \ForEach{$w\in \m{cons}(W)$} {
    $f[w]\gets 0$\;
  }
  $Q \gets \{f\}$ \;
  \While{$Q\ne\varnothing$} {
    $f \gets$ pick and delete least cost front in $Q$ \;
    $P \gets \m{toCN}_=(W,f)$ \;
    $sol\gets \m{solveCN}(P)$ \;
    \eIf{$sol\ne\bot$} {
      \Return sol
    }
    {
      \ForEach{$w\in \m{cons}(W)$} {
        \If{$f[w]\ne nbLayers(w)-1$} {
          $f' \gets f$ \;
          $f'[w] \gets f'[w] + 1$ \;
          \If{$\m{cost}(f')\ne k$} {         
            $Q \gets Q \cup f'$ \; 
          }
        }
      }
    }
  }
  \caption{completeSearchNoMUC($W$: WCN)}
  \label{csnomuc}
\end{function}

The priority queue $Q$ used in this algorithm is a variant of the
usual priority queues because it must ensure that fronts are extracted
in increasing cost order but also that a given element is not inserted
twice in the queue.  For example, if we consider a network composed of
two constraints each having at least two layers, the first front will
be the array $<0,0>$. During the first iteration, we shall insert in
$Q$ the neighbors $<1,0>$ and $<0,1>$. When these fronts are
extracted from the queue, they both generate the same neighbor $<1,1>$
that must be inserted only once in $Q$ to avoid redundant
computations.  This particular queue is obtained by slightly modifying
a classical implementation of a priority queue with a binomial heap.

One can easily check that Algorithm \ref{csnomuc} enumerates all
possible fronts when $\m{solveCN}$ never finds a solution.  Indeed, if
we ignore costs, this algorithm performs a breadth-first search
enumerating fronts. Proposition \ref{queueproperty} also
guarantees that this algorithm does not loop endlessly and that fronts
are enumerated in increasing cost order.

\begin{proposition}
\label{queueproperty}
The following properties are guaranteed by \ref{csnomuc}: 
\begin{itemize}
\item[A)] fronts are enumerated in increasing cost order;
\item[B)] once a front $f$ is extracted from the priority queue, it can never be inserted again.
\end{itemize}
\end{proposition}

\begin{proof}
Let $f_1$ be a first front extracted from the queue and $f_2$ a front
extracted from the queue after the extraction of $f_1$.

A) Two cases are possible. Either (A.1) $f_2$ was already present in the queue when $f_1$ has been extracted or (A.2) $f_2$ has been inserted in the queue after the extraction of $f_1$.
In case (A.1), the priority queue ensures that $cost(f_1)\le cost(f_2)$. In case (A.2), $f_2$ comes from a front $f$ such that $f\rightarrow_* f_2$ and either $f=f_1$ or $f$ 
was present in the queue when $f_1$ has been extracted. In both cases, $cost(f_1)\le cost(f)$. As $f\rightarrow_* f_2\Rightarrow cost(f)< cost(f_2)$, we conclude that $cost(f_1)<cost(f_2)$.
Therefore, we obtain that in any case $cost(f_1)\le cost(f_2)$.

B) Let us assume that $f_1=f_2$. As the queue ensures that it never contains two identical fronts at a given time, we conclude that we are necessary in case (A.2).
However, we proved that in this case, $cost(f_1)< cost(f_2)$, which contradicts $f_1=f_2$. \hfill $\qed$
\end{proof}

It should be emphasized that in our approach, unary constraints are considered as normal constraints. Moreover, because only one layer of a constraint is considered at a given time, the CN generated
are typically much smaller than the original WCN. One can expect that checking the satisfiability will be quite simple (hence fast) in most cases.
One last observation here is that this algorithm exploits a form of abstraction considering that, from a global point of view, there is no reason to distinguish tuples having the same cost. 

To conclude this section, the drawback of this first version is the enumeration of all possible combinations of layers. In the worst-case, it enumerates a number of fronts equal to the product
of the number of layers of each constraint, which corresponds to $\prod_{w\in cons(W)}nbLayers(w)$. Depending on the number of variables and constraints, this complexity may be much greater compared to a 
classical branch and bound approach (i.e., when $\prod_{w\in cons(W)}nbLayers(w) \ggg \prod_i|dom(x_i)|$). This is explained by the fact that, in our approach, the same instantiation can be
explored several times in different tests of satisfiability. Nevertheless, even in this case, this approach can be effective if the optimum is low.

\subsection{Exploiting MUCs}

The complexity of the previous algorithm can be reduced considerably in practice by noticing that it is useless to move from a layer to the next one for a constraint which does not participate in the unsatisfiability of
the CN. The idea is therefore to identify a MUC and to move to the next layer only for the constraints belonging to this MUC.

In the worst-case, the MUC extracted contains all the constraints of the problem and therefore the complexity in the worst-case remains unchanged. However in practice, on concrete instances,
MUCs are often small. Hence, the generated neighbourhood is much smaller and the practical complexity is significantly reduced.

\subsubsection{Complete Approach}

Function \ref{cs} (inspired by \cite{ABL_NAWPM,FM_SPMP,SP_AMSUC}) presents the algorithm modified in order to take profit of MUCs in a complete approach. The first steps of the algorithm
are identical to those of \ref{csnomuc}. When the current CN is unsatisfiable, a MUC is identified, and we shall progressively allow the higher layers in the constraints of the MUC.
More precisely, the algorithm enumerates all the possible ways to relax this MUC. For each constraint of the MUC (obtained by a call to $\m{cons}(W,M)$), the algorithm generates a successor $f'$
from the current front which differs only by the incrementation of the allowed layer for this constraint. This new front is inserted in $Q$ at a position depending on the value of
$cost(f')$.

\begin{function}
  \ForEach{$w\in \m{cons}(W)$} {
    $f[w]\gets 0$\;
  }
  $Q \gets \{f\}$ \;
  \While{$Q \ne \varnothing$}  {
    $f \gets$ pick and delete least cost front in $Q$ \;
    $P \gets \m{toCN}_=(W,f)$ \;
    $sol \gets \m{solveCN}(P)$ \;
    \eIf{$sol \neq \bot$} {
      \Return sol 
    }
    {
      $M \gets \m{extractMUC}(P)$ \;
      \ForEach{$w\in \m{cons}(W,M)$} {
        \If{$f[w]\ne nbLayers(w)-1$} {
          $f' \gets f$ \;
         $f'[w] \gets f'[w] + 1$ \;
          \If{$\m{cost}(f')\ne k$} {         
            $Q \gets Q \cup \{f'\}$ \;
          }           
        }
      }
    }
  }

  \caption{completeSearch($W$: WCN)}
  \label{cs}

\end{function}

\begin{proposition}
\label{treillis}
Algorithm \ref{cs} is complete.
\end{proposition}

\begin{proof}
To simplify the proof, we shall use the term  ``front'' to refer implicitly to the CN associated with a front $f$ by $\m{toCN}_=(W,f)$.
Let $f_u$ be the front which is identified as unsatisfiable and $M$ the MUC extracted from this front. The only difference between \ref{cs} and \ref{csnomuc} is
that when a MUC is identified, we restrict the search by relaxing this MUC prior to any other operation.

If the WCN has no solution, restricting the search in this way obviously does not cause the loss of any solution. We can therefore only consider the case where the CN admits at least one solution $S$.
Let $f_S$ be the front corresponding to a solution $S$. The algorithm \ref{csnomuc} ensures that there always exists a front $f$ in the priority queue such that $f\rightarrow_*f_S$.
This is especially true for the first front inserted in the queue $\bot$ and when this property is satisfied for a front extracted from the queue, it is satisfied for at least one of its
direct successors (by definition of $\rightarrow_*$). Therefore, by recurrence, this property is always ensured.

If $f_u\not\rightarrow_* f_s$, the proposed restriction does not result in the loss of solution $S$. Therefore, we can further assume that $f_u\rightarrow_* f_s$. Hence, $\forall w \in cons(W), f_s[w]\ge f_u[w]$.
Moreover, there necessarily exists a constraint $w_i\in cons(W,M)$ such that $f_s[w_i]> f_u[w_i]$, otherwise $f_S$ will still be unsatisfiable. Let $f'$ be the front defined by $f'[w_i]=f_u[w_i]+1$
and $\forall w \in cons(W)$ such that $w\ne w_i$, $f'[w]=f_u[w]$. $f'$ is a front which is generated by the algorithm and by construction $f'\rightarrow_* f_s$. \hfill $\qed$
\end{proof}

Figure \ref{contreExempleRelachementUnique} presents an example where one cannot simply relax a MUC in only one way, even if it was locally the optimal relaxation, without losing the optimal solution.
In this example, a WCN is composed of three constraints $w_x$, $w_{xy}$ and $w_y$. The first front selects the layers having a cost equal to 0, and therefore only retains  
$\{(x,a)\}$, $\{(x,a),(y,b)\}$ and $\{(y,a)\}$ as allowed tuples. Of course, in this case, the constraints $w_{xy}$ and $w_y$ constitute a MUC.
There are two ways to relax this first MUC. The first relaxation switches to the next layer of $w_{xy}$ (it is locally the best choice since it has a cost of 5) whereas the second one switches to the next layer of $w_y$.
Considering the first case, the new front obtained retains $\{(x,a)\}$, $\{(x,c),(y,a)\}$ and $\{(y,a)\}$ as allowed tuples. This time, the MUC contains the constraints $w_x$ and $w_{xy}$.
We can either relax $w_{xy}$ to obtain a solution with a cost of 100, or relax $w_x$ to obtain another MUC that can be relaxed in two different ways to obtain either a solution
having a cost of 105, or a solution having a cost of 110.

The second relaxation of the first MUC consists in switching to the
next layer of $w_y$. In this case, we directly obtain a solution
having a cost of 10, which is the optimum.
This clearly shows that all relaxations of a MUC are required.

\begin{figure}[h]
\begin{center}
  \setlength{\tabcolsep}{3pt}
  \small
  \subfigure[$w_x$]
  {
    \begin{tabular}{c|c}
      cost & tuples \\
      \hline
      \hline
      100 & \{(x,c)\} \\
      \hline
      10 & \{(x,b)\} \\
      \hline
      0 & \{(x,a)\} \\
    \end{tabular}
  }
  \subfigure[$w_{xy}$]
  {
    \begin{tabular}{c|c}
      cost & tuples \\
      \hline
      \hline
      100 & default \\
      \hline
      5 & \{(x,c),(y,a)\} \\
      \hline
      0 & \{(x,a),(y,b)\} \\
    \end{tabular}
  }
  \subfigure[$w_y$]
  {
    \begin{tabular}{c|c}
      cost & tuples \\
      \hline
      \hline
      100 & \{(y,c)\} \\
      \hline
      10 & \{(y,b)\} \\
      \hline
      0 & \{(y,a)\} \\
    \end{tabular}
  }
\end{center}
  \caption{Example where it is required to enumerate all relaxations of a MUC to reach the optimum.}
  \label{contreExempleRelachementUnique}
\end{figure}

\subsubsection{Greedy Approach}

In this part, we present an incomplete version of our approach
based on the identification and relaxation of MUCs in order to solve
WCNs.

From a WCN $W$, Function \ref{alg:gmr} returns a solution for a CN derived from $W$.
First the structure $front$ is initialised to $0$ which corresponds to the first layer of each constraint of the WCN.

From a WCN $W$ and a front $f$, we extract a CN and this network is then solved.
If a solution is found, then it is returned by the function \ref{alg:gmr} (line $7$).
If the network has been proved unsatisfiable, it is necessary to relax some constraints of the WCN.
To achieve this, the algorithm extracts a WCN $W'$ from a computed MUC (lines $9$ and $10$).
The front $f$ is then updated by Function \ref{alg:h2relax} (line $11$). This procedure will relax one (or several constraints) of $W'$ in order to break the MUC associated with $W'$.
This process loops until a solution is found for the constraint network associated with $W$.

\begin{function}[t!]
  \ForEach{$w \in \m{cons}(W)$} {
    $f[w]\gets 0$ \;
   }
  \Repeat{$sol\ne\bot$} {
    $P \gets \m{toCN}_\le(W,f)$ \;
    $sol \gets \m{solveCN}(P)$ \;
    \eIf{$sol \neq \bot$} {
      \Return sol \;
    }
    {
      $M \gets \m{extractMUC}(P)$ \;
      $W' \gets \m{restrict}(W,M)$ \;
      $f \gets \m{relax}(W',f)$ \;
   }
  }

  \caption{incompleteSearch($W$: WCN)}
  \label{alg:gmr}

\end{function}
 \begin{function}[t]
   $Q_{loc} \gets \{f \}$ \;
   \While{$Q_{loc} \neq \varnothing$} {
      $f \gets$ pick and delete least cost front in $Q_{loc}$ \;
      $P \gets \m{toCN}_\le(W,f)$ \;
     \uIf{$\m{solveCN}(P) \neq \bot$} {
       \Return{$f$} \;
     }
     \Else {
       $M \gets \m{extractMUC}(P)$ \;
       \ForEach{$w \in \m{cons}(W,M)$} {
         \If{$f[w] \ne nbLayers(w)-1$} {
           $f' \gets f$ \;
            $f'[w] \gets f'[w] + 1$ \;
           \If{$\m{cost}(f') \ne k$} {
             $Q_{loc} \gets Q_{loc} \cup \{f'\}$ \;
           }
         }
       }
     }
   }
   \caption{ relax($W$: WCN, f: front): front}
   \label{alg:h2relax}
 \end{function}
 
%
%
%
%
%
%

Function $relax(W,f)$ updates the front $f$ in order to allow new layers. The idea is to increment $f[w]$ for at least one constraint $w$, 
in such a way as to make the MUC satisfiable. In other words, for at least one of the constraints, we accept an increase of cost for this constraint.
In the greedy approach used here, we keep doing this until the MUC is broken.

Algorithm \ref{alg:h2relax} uses a local priority queue  $Q_{loc}$ which is initialized with the front $f$.
While $Q_{loc}$ is not empty (line $2$), we extract from the queue the front $f$ having the lowest cost $cost(f)$.
Function $\m{toCN}_\le$ builds a constraint network from both the WCN and the front $f$ previously selected (line $4$).
If this constraint network is satisfiable, $f$ initially given in parameter is updated and returned by the algorithm \ref{alg:h2relax}.
On the contrary, if the constraint network is unsatisfiable, it means that the relaxation is not sufficient and then the algorithm updates the queue of fronts by generating all the neighbors of $f$. 
For each constraint belonging to the WCN (obtained from the MUC), we generate a new front,
different from the initial one, by incrementing the layer of one single constraint.

Note that in the different calls to the function \ref{alg:h2relax}, it is sufficient to work only on a subset of the constraints.
Indeed, the constraints that can be relaxed are those belonging to the WCN $W$ associated with the MUC ($restrict(W,M)$).
In practice, we only work with the subset of the front which corresponds to contraints of the MUC, which saves space (for simplicity, this is not detailed in the algorithm).

In the incomplete version, we transform a WCN into a CN using function  $\m{toCN}_\le$ instead of $\m{toCN}_=$ which is used in the complete version.
Indeed, in the complete version, all possible relaxations of a MUC are considered and the fronts are enumerated in increasing cost order. Hence, when we test the satisfiability
of a CN associated with a front $f$, we know that all the CN associated with fronts $f'$ such that $f'\rightarrow_* f$ have already been proved unsatisfiable.
In the complete version, $\m{toCN}_=$ extracts a CN composed of the current layers of a front.
In the incomplete version, all possible relaxations of a MUC are not tried. Therefore, we cannot ensure that, when we test the satisfiability of a CN associated with a front $f$,
all the CNs associated with fronts $f'$ such that $f'\rightarrow_* f$ have been tried and already proved unsatisfiable.
Function $\m{toCN}_\le$ therefore extracts a CN from both the current and inferior layers of a front.

This approach doesn't guarantee the identification of an optimum solution. The reason is that all the relaxations of MUCs are not considered, contrary to the complete version 
(see Example \ref{contreExempleRelachementUnique}). Indeed, we only identify the first relaxation which restores satisfiablity of a MUC. 

\section{Experimental Results} \label{sec:experimental}

In order to show the practical interest of our approach, we have performed experiments using a computer with processors Intel(R) Core(TM) i7-2820QM CPU 2.30GHz.
Our greedy approach, noted $\m{GMR}$, has been compared with two complete approaches with cost transfer algorithms enforcing EDAC, proposed by both our solver $AbsCon$ and the solver $ToulBar2$.
A time-out of $600$ seconds was set per instance.
The total CPU time necessary to solve each instance is given as well as the upper bound found (UB).
If the execution of the algorithm is not yet finished before the time-out (CPU time greater than $600$ seconds), the bound found at the end of the $600$ seconds is given.
Table \ref{tab:spot5} provides the results obtained for the serie $spot5$ (except trivial instances) composed of constraints of arity less than or equal to $3$.
Table \ref{tab:celar}  provides the results obtained for the serie $celar$.

\begin{table}[h!]
\begin{center}
{\small
\begin{tabular}{ccrrr}
  \multicolumn{2}{c}{~ }  & \multicolumn{2}{c}{AbsCon}   & \multicolumn{1}{c}{ToulBar2}  \\

$Instances$& & \multicolumn{1}{c}{GMR}  & \multicolumn{1}{c}{EDAC}
 &\multicolumn{1}{c}{EDAC} \\
\toprule

\vspace{-0.4cm} \multirow{1}{3cm}{ }  &\multirow{1}{1cm}{ }&\multirow{1}{1.7cm}{ }&\multirow{1}{1.7cm}{ }&\multirow{1}{1.7cm}{ } \\

\rowcolor{vlgray} $spot5$-$42$ & CPU & $9.18$ & $>600$ & $>600$ \\
 & UB & $ {\bf 161,050}$ &$ {\bf 161,050}$ &  ${\bf 161,050}$ \\
\midrule
\rowcolor{vlgray}  $spot5$-$404$ & CPU & $4.99$ & $>600$  &$217$ \\
  & UB & $118$ & ${\bf 114}$   &${\bf 114}$ \\
\hline
\rowcolor{vlgray} $spot5$-$408$ & CPU & $9.85$ & $>600$  &$>600$ \\
  & UB & ${\bf 6,235}$ & $8,238$  & $6,240$ \\
\midrule
\rowcolor{vlgray} $spot5$-$412$ & CPU & $18.8$ & $>600$ &$>600$ \\
  & UB & ${\bf 33,403}$ & $43,390$  &$37,399$ \\
\midrule
\rowcolor{vlgray} $spot5$-$414$ & CPU & $37.7$ & $>600$  &$>600$ \\
  & UB & ${\bf 40,500}$ & $56,492$  &$52,492$ \\
\midrule
\rowcolor{vlgray}$spot5$-$503$ & CPU & $6.33$ & $>600$  &$>600$ \\
  & UB & $12,125$ & $13,119$  &${\bf 12,117}$ \\
\midrule
\rowcolor{vlgray} $spot5$-$505$ & CPU & $12$ & $>600$  &$>600$ \\
 & UB & ${\bf 22,266}$ & $28,258$ &$25,268$ \\
\midrule
\rowcolor{vlgray}$spot5$-$507$ & CPU & $22.3$ & $>600$  &$>600$ \\
  & UB & ${\bf 30,417}$ & $37,429$  &$37,420$ \\
\midrule
\rowcolor{vlgray} $spot5$-$509$ & CPU & $32.2$ & $>600$ &$>600$ \\
  & UB & ${\bf 37,469}$ & $48,475$  &$46,477$ \\
\midrule
\rowcolor{vlgray}$spot5$-$1401$ & CPU & $76.5$ & $>600$ &$>600$ \\
  & UB & ${\bf 483,109}$ & $513,097$ &$516,095$ \\
\midrule
\rowcolor{vlgray} $spot5$-$1403$ & CPU & $142.5$ & $>600$  &$>600$ \\
  & UB & ${\bf 481,266}$ & $517,260$  &$507,265$ \\
\midrule
\rowcolor{vlgray}$spot5$-$1407$ & CPU & $552.6$ & $>600$  &$>600$ \\
  & UB & ${\bf 492,614}$ & $517,623$ &$507,633$ \\
\midrule
\rowcolor{vlgray} $spot5$-$1502$ & CPU & $3.91$ & $0.98$  &$0.02$ \\
 & UB & $28,044$ &$ {\bf 28,042}$ &${\bf 28,042}$ \\
\midrule
\rowcolor{vlgray}$spot5$-$1504$ & CPU & $67.6$ & $>600$ &$>600$ \\
  & UB & ${\bf 175,311}$ & $204,314$   &$198,318$ \\
\midrule
\rowcolor{vlgray} $spot5$-$1506$ & CPU & $338$ & $>600$ &$>600$ \\
  & UB & ${\bf 378,551}$ & $426,551$  &$399,568$ \\
\bottomrule
\end{tabular}
}
\end{center}
\caption{CPU time in seconds and bound found before the time-out for instances \p{spot5}. \label{tab:spot5}}

\end{table}

\begin{table}[h!]
\begin{center}
{\small
\begin{tabular}{ccrrr}
  \multicolumn{2}{c}{~ }  & \multicolumn{2}{c}{AbsCon}   & \multicolumn{1}{c}{ToulBar2}  \\

$Instances$& & \multicolumn{1}{c}{GMR}  & \multicolumn{1}{c}{EDAC}
 &\multicolumn{1}{c}{EDAC} \\
\toprule

\vspace{-0.4cm} \multirow{1}{3cm}{ }  &\multirow{1}{1cm}{ }&\multirow{1}{1.7cm}{ }&\multirow{1}{1.7cm}{ }&\multirow{1}{1.7cm}{ } \\

\rowcolor{vlgray} $graph$-$05$ & CPU & $16.6$ & $>600$   & $0.62$ \\
  & UB & ${\bf 221}$ & $4,645$ & ${\bf 221}$ \\
\midrule
\rowcolor{vlgray} $scen$-$06$ & CPU & $88.5$ & $>600$ &$485.4$ \\
  & UB & $3,616$ & $12,013$ &${\bf 3,389}$ \\
\midrule
\rowcolor{vlgray} $scen$-$06$-$16$ & CPU & $60.6$ & $>600$ & $237.7$ \\
  & UB & $4,149$ & $11,286$  &${\bf 3,277}$ \\
\midrule
\rowcolor{vlgray} $scen$-$06$-$18$ & CPU & $59.2$ & $>600$ & $102.4$ \\
  & UB & $3,640$ & $8,723$  &${\bf 3,263}$ \\
\midrule
\rowcolor{vlgray} $scen$-$06$-$20$ & CPU & $68.5$ & $>600$ & $67.9$ \\
  & UB & $3,402$ & $8,594$  &${\bf 3,163}$ \\
\midrule
\rowcolor{vlgray} $scen$-$06$-$30$ & CPU & $32.6$ & $177.7$ & $1.10$ \\
  & UB & $2,208$ & ${\bf 2,080}$  &${\bf 2,080}$ \\
\midrule
\rowcolor{vlgray} $scen$-$07$ & CPU & $209.9$ & $>600$  &$>600$ \\
  & UB & ${\bf 426,423}$ & $31,230K$  &$505,731$ \\
\midrule
\rowcolor{vlgray} $scen$-$07$\_$10000$\_$30r$ & CPU & $28.7$ & $>600$ &$>600$ \\
  & UB & ${\bf 270K}$ & $17,000K$  &$1,500K$ \\
\midrule
\rowcolor{vlgray} $celar6$-$sub2$ & CPU & $23.3$ & $>600$  &$4.74$ \\
  & UB & $2,927$ & ${\bf 2,746}$ &${\bf 2,746}$ \\
\midrule
\rowcolor{vlgray} $celar6$-$sub3$ & CPU & $26.6$ & $>600$  &$15.7$ \\
  & UB & $3,271$ & $3,279$ &${\bf 3,079}$ \\
\midrule
\rowcolor{vlgray} $celar6$-$sub4$ & CPU & $41.5$ & $>600$  &$28.7$ \\
  & UB & $3,704$ & $5,178$ &${\bf 3,230}$ \\
\midrule
\rowcolor{vlgray} $celar7$-$sub2$ & CPU & $60.8$ & $>600$  &$17.8$ \\
  & UB & $283,955$ & $ 252,436$ &${\bf 173,252}$ \\
\midrule
\rowcolor{vlgray} $celar7$-$sub3$ & CPU & $48.7$ & $>600$  &$93.4$ \\
  & UB & $414,161$ & $1,342K$ &${\bf 203,460}$ \\
\midrule
\rowcolor{vlgray} $celar7$-$sub4$ & CPU & $57$ & $>600$  &$221$ \\
  & UB & $272,945$ & $ 302,541$ &${\bf 242,443}$ \\
\bottomrule

\end{tabular}
}
\end{center}
\caption{CPU time in seconds and bound found before the time-out for instances \p{celar}. \label{tab:celar}}

\end{table}

On instances of $spot5$, we observe that our approach provides better bounds than those obtained by the two complete approaches, except for the instances $spot5$-$404$, $spot5$-$503$ and $spot5$-$1502$.
Beyond the fact that our approach is not complete, this can be explained by the fact that these instances are relatively easy (involving less than $1,400$ constraints) and cost transfer algorithms
reach a better bound before the time-out. However, it is interesting to note that the bound found by our approach is not so different from these obtained by complete approaches,
and this in a very short time. About other instances, considered as more difficult (in particular, more than $13,000$ constraints for the instances $spot5$-$1403$, $spot5$-$1407$ and $spot5$-$1506$),
we observe that our greedy approach provides a better bound than those of complete approaches in a time shorter than the time-out.

One can note that $\m{GMR}$ is not as competitive as $ToulBar2$ on some instances of $celar$, $scen$ and $graph$.
This is partly due to the fact that $ToulBar2$ uses a constraint decomposition technique. 
Indeed, comparing $\m{GMR}$ with a classical cost transfer algorithm (such as in our $AbsCon$ solver), the approach described in this paper outperforms the EDAC version of $AbsCon$.
However, on the $graph$-$05$ instance, even if $\m{GMR}$ is not so fast than $ToulBar2$, our approach is able to find the optimum upper bound.   

 
The preliminary experiments we have conducted on the complete approach show that our current implementation is not yet competitive. It clearly appears that MUCs must be identified in an incremental way so that the computational effort for a given front benefits to the others. We are currently exploring the use of dynamic CSP algorithms to incrementally test the satisfiability of CNs and identify MUCs.

\section{Conclusion}

In this paper, we have proposed an original approach for solving weighted constraint networks through successive resolutions of hard constraint networks.
More precisely, CNs are obtained from WCNs by selecting tuples with a given cost.
These CNs are enumerated in increasing cost order until a solution is found.
To improve the complexity in practice of our approach, we identify a minimal unsatisfiable core for each unsatisfiable constraint network, in order to focus the cost increase on the sole constraints of the MUC.
The approach is declined in both a complete and an incomplete, greedy algorithm.
We have shown that our method based on MUCs extraction can be used in practice: the greedy algorithm obtains results which are comparable with other state of the art approaches.
However, the complete algorithm is not yet competitive but we have identified several reasons.
We are working on a new version using dynamic CSP algorithms in order to perform incremental computations.
To conclude, we would like to emphasize that the proposed method allows a simple and natural integration of the CSP, (weighted) Max-CSP and WCSP frameworks.
This can be done by generalizing in a straightforward way the construction of CNs generated during the resolution.

\section*{Acknowledgments}

This work has been supported by both CNRS and OSEO within the ISI project 'Pajero'.



\end{document}